%% file: camera-ready.tex
\documentclass{article}

    \PassOptionsToPackage{numbers, compress}{natbib}

 \usepackage[main, final]{neurips_2026}

\usepackage[utf8]{inputenc} 
\usepackage[T1]{fontenc}    
\usepackage[colorlinks=true, urlcolor=utburnt, citecolor=utburnt, linkcolor=utburnt]{hyperref}       
\usepackage{url}            
\usepackage{booktabs}       
\usepackage{amsfonts}       
\usepackage{nicefrac}       
\usepackage{microtype}      
\usepackage{xcolor}         

\usepackage{algorithm}
\usepackage{kotex}
\usepackage{multirow}
\usepackage{makecell} 
\usepackage{graphicx} 
\usepackage{adjustbox}
\usepackage{float}
\usepackage{wrapfig}
\usepackage{colortbl}
\usepackage[endLComment=, italicComments=false]{algpseudocodex}
\usepackage{algorithm}

\usepackage[most]{tcolorbox}

\usepackage{amsmath}
\usepackage{amssymb}
\usepackage{mathtools}
\usepackage{amsthm}

\usepackage{pifont}
\usepackage{xspace}
\usepackage{dsfont}

\usepackage{amsmath,amsfonts}

\definecolor{utburnt}{HTML}{BF5700}
\definecolor{utgray}{HTML}{333F48} 
\definecolor{myred}{HTML}{F54254}
\definecolor{myorange}{HTML}{FFB135}
\definecolor{mygreen}{HTML}{10BD35}
\definecolor{myblue}{HTML}{598BE7}
\definecolor{mypurple}{HTML}{9A1C6B}
\definecolor{plgray}{HTML}{999999}

\usepackage{lipsum}
\usepackage{comment}

\newcommand{\STAB}[1]{\begin{tabular}{@{}c@{}}#1\end{tabular}}

\theoremstyle{plain}
\newtheorem{theorem}{Theorem}[section]
\newtheorem{proposition}[theorem]{Proposition}

\theoremstyle{definition}

\theoremstyle{property}

\newtheorem{assumption}[theorem]{Assumption}
\theoremstyle{remark}

\title{Test-time Multi-agent Coordination by \\ Decomposed Value Gradient Flow}

\author{%
  Dongsu Lee \qquad 
  Haoran Xu \qquad
  Amy Zhang
  \\
  The University of Texas at Austin
  \\
  \texttt{dongsu.lee@utexas.edu}
  \\
  \color{utburnt}\texttt{\url{https://github.com/DongsuLeeTech/scout.git}}
}

\begin{document}

\maketitle

\begin{abstract}
Offline multi-agent reinforcement learning (MARL) faces a persistent trade-off. Expressive generative policies can represent multi-modal coordination in the data, but cannot distinguish high-value regions, while value-optimized policies exploit the learned Q-function but collapse the multi-modal into a single dominant mode. A single agent's mode collapse can break joint coordination, and simultaneous drift across agents can push the joint policy into unseen regions of the action space. We propose \textbf{scalable coordination via optimal unified transport (SCOUT)}, the first offline MARL framework to combine a generative foundation model with a learned value function through test-time action refinement. SCOUT trains two decoupled components: a flow-matching behavioral prior and a decomposed value function. At test-time, it transports behavioral samples toward high-value regions via Stein variational gradient descent. The number of transport steps controls adaptive test-time scaling, replacing a fixed regularization coefficient. Under the individual-global-max (IGM) principle, we prove a single-term KL bound on the joint soft-value gap that vanishes as transport converges, with an irreducible additive residual proportional to the IGM violation. Empirically, SCOUT achieves the best average performance across discrete and continuous offline MARL benchmarks and yields performance improvements in all offline-to-online configurations.
\end{abstract}

\input{figure/overview}

\setcounter{tocdepth}{1}
\addtocontents{toc}{\protect\setcounter{tocdepth}{-1}}

\section{Introduction}
\label{sec: intro}
Offline MARL is motivated by the need to learn coordinated multi-agent behaviors in safety-critical or cost-constrained domains where active data collection is infeasible or expensive, \textit{e.g.}, autonomous driving~\cite{shalev2016safe}, logistics~\cite{tamagawa2010evaluating}, and robotics~\cite{ota2006multi}. In this setting, the learning objective is twofold. Agents must learn to coordinate their behavior from heterogeneous and suboptimal demonstrations, while avoiding the distributional shift that arises when the learned policy visits state-action regions unsupported by the data. Early offline MARL methods~\cite{yang2021believe, pan2022plan, wang2023offline} rely on simple policy classes (\textit{i.e.}, Gaussian), which struggle to represent multi-modal and coordinated joint action distributions. When offline data contains a mixture of diverse coordination strategies, a unimodal policy often collapses onto a single mode and produces out-of-distribution joint actions that no agent in the data ever executed together.

To address this expressiveness gap, recent works adopt powerful generative models, \textit{including} diffusion models~\cite{zhu2024madiff, li2025dof} and flow matching  models~\cite{li2025om2p, lee2025multi}, as policy representations in offline MARL. Although these models improve the representational power of policy, their iterative sampling process incurs a high inference cost. MAC-Flow~\cite{lee2025multi} and OM$^2$P~\cite{li2025om2p} mitigate this via one-step distillation and mean-flow generation, respectively. However, improved expressiveness and faster inference do not resolve a more fundamental limitation. Since the outcome policy is optimized against a learned Q-function, it inherently tends to collapse diverse coordination modes into a single or blurred one. 

Methods that parameterize a policy and optimize it to maximize a learned Q-function must balance value improvement against staying close to the behavioral distribution. A fixed regularization coefficient controls this trade-off. If it is too large, the policy stays too close to suboptimal data; if it is too small, the policy exploits errors in the learned Q-function and produces OOD actions. Although this tension is well-known in single-agent offline RL~\cite{fujimoto2019off, kumar2020conservative}, the multi-agent setting introduces the following compounding difficulty. Training-time value maximization deforms each agent's policy away from the behavioral prior, and the regularization coefficient bounds how far each marginal may drift, but cannot bound the compound effect on the joint. Even when every marginal remains within its individual support, the joint policy can place mass on action combinations that no trajectory in the data ever visited (\textit{i.e.}, regions where the learned Q-function is least reliable). In fact, the data may support multiple distinct joint modes of comparably high estimated value, but committing each marginal to a single greedy mode at training time discards the others before the agent has any way to tell them apart. Once new interaction data arrives, the policy can no longer use it to discriminate which mode is truly higher-value, because the alternatives have already been suppressed.

A natural alternative is to decouple two objectives, motivated by test-time training~\cite{park2024value, xureinforcement, lee2026policy}. Instead of learning a policy that directly maximizes value, we keep the behavioral policy as a foundational prior and defer value-based refinement to test time. At deployment, we sample action candidates from this prior and then move them toward high-value regions using value gradients. The extent of this adjustment acts as adaptive control, replacing the need for a fixed regularization coefficient.

In this work, we ask a question not yet studied by prior works:
\begin{center}
    \textit{How do we avoid explicit policy parameterization for value maximization, and instead recover coordination at test time by directly transporting behavioral samples toward high-value regions?}
\end{center}
We answer this question with the proposed solution, \textbf{SCOUT (scalable coordination via optimal unified transport)}, an offline MARL framework that combines flow-matching behavioral priors with decomposed value gradient for test-time action refinement. Our key contributions are:
\begin{itemize}
    \item We propose a policy-free value maximization for offline MARL that separates behavioral modeling from value improvement, replacing training-time regularization tuning with a test-time refinement budget.
    \item We derive decomposed value gradient transport under the IGM principle, enabling fully decentralized test-time refinement that recovers inter-agent coordination with a bounded approximation gap.
    \item Empirically, test-time refinement strengthens coordination without retraining, achieving the best average performance in continuous and discrete domains. Notably, unlike existing flow-based MARL, our method yields improvements in all offline-to-online RL experiments.
\end{itemize}

\section{Related Works}
\textbf{Offline MARL.} Several works have explored the data-driven sequential decision-making problem, \textit{including} offline RL and MARL, which aims to extract a return-maximizing policy within the support of an offline dataset. Specifically, offline MARL requires learning coordinated policies for multiple agents. This is not trivial because it requires modeling dependencies among agents in a combinatorially large joint action space. The simplest baselines are multi-agent extensions of single-agent offline RL methods, either by introducing conservatism in the value function~\cite{fujimoto2019off, kumar2020conservative} or the actor~\cite{fujimoto2021minimalist}, or by performing Bellman updates without querying out-of-distribution actions~\cite{kostrikov2021offline}. To further promote coordination, more advanced works are often combined with additional multi-agent design choices, such as value decomposition~\cite{pan2022plan, liu2025offline}, policy optimization with a centralized critic~\cite{yang2021believe, wang2023offline}, or a centralized world model~\cite{barde2023model}. However, these methods often struggle to capture the multi-modal distributions present in multi-agent offline datasets because of their limited policy expressiveness. To address this issue, recent work has adopted more expressive generative policies, including an autoregressive model~\cite {meng2023offline}, diffusion models~\cite{zhu2024madiff}, and flow matching~\cite{lee2025multi}. However, these methods couple value learning with the generative process itself: they either anchor the policy via BC-like losses~\cite{wang2022diffusion, park2025flow, lee2025multi, lee2026towards}, which produces overly conservative behavior, or backpropagate value gradients through iterative sampling \cite{ren2024diffusion, mark2024policy}, which is unstable. In contrast, the proposed solution bypasses explicit policy parameterization altogether. Instead, Stein variational gradient~\cite{liu2016stein, liu2017stein, xureinforcement} fields refine flow-generated action candidates toward high-value regions at test time. This enables adaptive test-time scaling while retaining generative expressiveness.

\textbf{Test-time action refinement.} Our main idea in this work is based on test-time refinement (\textit{i.e.}, $a\leftarrow a+\nabla_a Q(s,a)$)~\cite{park2024value}, which refines actions from a base generative policy at inference time rather than during training. This paradigm has been explored along several directions in single-agent RL. Rejection sampling methods generate multiple action candidates and select the highest-valued one~\cite{hansen2023idql, li2025reinforcement}. Gradient-based approaches directly optimize sampled actions via Q-function gradients~\cite{mark2024policy, xureinforcement, li2026qlearningadjointmatching}. Residual policies instead train a separate network to modify a base policy in noise or action space~\cite{singh2020parrot, wagenmaker2025steering}. Complementary to these, optimal transport (OT) offers a principled framework for moving probability mass toward high-value regions~\cite{dabney2018distributional, dadashi2020primal}. More recently, particle-based methods adopt an OT viewpoint to transport samples toward value-preferred regions~\cite{asadulaev2024rethinking}. However, these approaches remain limited to single-agent settings. The proposed solution is motivated by this direction but specifically designed to solve the multi-agent coordination problem. This avoids explicit policy parameterization and naturally handles the combinatorial joint action space.

\section{Preliminaries}
\textbf{Problem setting.} We consider cooperative MARL problems modeled as a decentralized partially observable Markov decision process (Dec-POMDP)~\cite{bernstein2002complexity}, defined by the tuple $\mathcal{M}=\langle \mathcal{I}, \mathcal{S}, \{\mathcal{O}_i\}, \{\mathcal{A}_i\}, \mathcal{T}, \{\Omega_i\}, \{r_i\}, \gamma \rangle$, where $\mathcal{I}=\{1, 2, \ldots, I\}$ is a finite set of agents, $\mathcal{S}$ denotes the state space, $\mathcal{O}_i$ and $\mathcal{A}_i$ are the observation and action spaces of agent $i$, $\mathcal{T}(s' \mid s, \mathbf{a}): \mathcal{S} \times \mathcal{A}_1 \times \cdots \times \mathcal{A}_I \to \Delta(\mathcal{S})$ is state transition function over state $s$ and joint action $\mathbf{a} = [a_1, \ldots, a_I]$, $\Omega_i(o_i \mid s): \mathcal{S} \to \Delta(\mathcal{O}_i)$ is the observation function, $r_i(s, \mathbf{a}): \mathcal{S} \times \prod_{i} \mathcal{A}_i \to \mathbb{R}$ is an individual reward, and $\gamma \in [0,1)$ is the discount factor. In the offline setting, agents learn from a fixed dataset $\mathcal{D} = \{\tau^{(n)}\}_{n=1}^{N}$, where each trajectory $\tau^{(n)}$ is collected by unknown behavioral policies. The behavioral policy can be empirically defined as $\mu(\mathbf{a} \mid \mathbf{o}) \approx p_{\mathcal{D}}(\mathbf{a} \mid \mathbf{o})$ and its per-agent marginals as $\mu_i(a_i \mid o_i)$. We consider the centralized training and decentralized execution (CTDE)~\cite{yang2021believe, wang2023offline, lirevisiting, lee2025multi}.

Offline RL methods typically aim for Q-value maximization because optimizing cumulative return at a given MDP is infeasible without additional interaction~\cite{levine2020offline, kumar2020conservative}. To make this tractable in MARL, we decompose the global action-value function $Q_{\mathrm{tot}}(\mathbf{o}, \mathbf{a})$ into per-agent utilities $\{Q_i(o_i, a_i)\}_{i=1}^{I}$ under the individual-global-max (IGM) principle~\cite{son2019qtran, rashid2020monotonic} as follows: 
\begin{equation}
\arg\max_{\mathbf{a}}  Q_{\mathrm{tot}}(\mathbf{o}, \mathbf{a}) = \bigl(\arg\max_{a_1} Q_1(o_1, a_1), \ldots, \arg\max_{a_I} Q_I(o_I, a_I)\bigr).
\label{eq: igm}
\end{equation}
We adopt the additive decomposition throughout. Under this, each agent independently maximizes its own value, and the resulting joint action is optimal, \textit{i.e.}, $\max_{\Pi} \mathbb{E}_{\mathbf{o} \sim \mathcal{D}}[Q_{\mathrm{tot}}(\mathbf{o}, \Pi(\mathbf{o}))]= \sum_{i=1}^{I} \max_{\pi_i} \mathbb{E}_{o_i \sim \mathcal{D}}[Q_i(o_i, \pi_i(o_i))]$. However, direct maximization of a learned Q-function leads to OOD errors. Therefore, in practice, the per-agent policy extraction problem can be defined as follows:
\begin{equation}
    \pi_i^* = \arg\max_{\pi_i} \mathbb{E}_{o_i\sim\mathcal{D}, a_i \sim \pi_i(\cdot|o_i)}\bigl[Q_i(o_i, a_i)\bigr], \quad \text{s.t.} \quad \mathbb{E}_{o_i\sim\mathcal{D}}\Bigl[D\bigl(\pi_i(\cdot|o_i), \mu_i(\cdot|o_i)\bigr)\Bigr] \leq \epsilon,
    \label{eq: policy}
\end{equation}
where $D$ is a divergence measure, \textit{e.g.}, KL-divergence or $\ell_2$-norm, and $\epsilon$ is a threshold.

\textbf{Policy extraction.} The behavior-regularized objective function in~\eqref{eq: policy} can be optimized by two families of solutions in general~\cite{lee2026recipestableofflinemultiagent}, differentiated by what they focus mode-seeking and mode-covering.

For mode-seeking, a representative example is a parameterized gradient~\cite{fujimoto2021minimalist, tarasov2023revisiting}, which extracts a Gaussian policy by converting the constraint into a penalty, following \texttt{DDPG+BC} objective~\cite{pan2022plan, lee2026recipestableofflinemultiagent}:
\begin{equation}
    \max_{\pi_i} \mathbb{E}_{o_i \sim \mathcal{D}}\Bigl[\mathbb{E}_{a_i \sim \pi_i}\bigl[Q_i(o_i,a_i)\bigr] - \beta D(\pi_i(\cdot|o_i), \mu_i(\cdot|o_i))\Bigr].
    \label{eq: ddpgbc}
\end{equation}
An alternative for a mode-seeking family is to use distillation to compress the multi-step generative model into a one-step policy~\cite{ding2023consistency, park2025flow, lee2025multi}. However, these methods output Gaussian policies, which may be insufficient because they collapse into a single mode and lack expressivity. For expressive generative policies such as diffusion~\cite{ho2020denoising} or flow models~\cite{lipman2022flow}, the parameterized gradient requires backpropagation through multi-step iterations, which is expensive and numerically fragile~\cite{wang2022diffusion}.

A second family samples actions from $\mu$ and reweights them by values. A closed-form solution is $\pi^*(a|s) \propto \mu(a|s)\exp(Q(s,a)/\beta)$, approximated by \texttt{weighted BC}~\cite{peng2019advantage, wang2023offline, xu2023offline}:
$$\max_{\pi} \mathbb{E}_{o_i \sim \mathcal{D}, a_i \sim \mu_i}\!\bigl[\exp(Q_i(o_i,a_i)/\beta) \cdot \log \pi_i(a_i \mid o_i)\bigr].$$
While this has been viewed as overly conservative~\cite{wu2026invisibleleashrlvrescape}, it carries a distinct advantage for multi-agent coordination~\cite{lee2026recipestableofflinemultiagent}. A key implication is that preserving the coordinated strategies in the dataset is often more important than aggressive mode-seeking, as independent mode selection can break inter-agent compatibility. As a theoretically near-optimal solution for this problem, best-of-$N$ draws $N$ candidate actions from $\mu$ and returns the highest-valued one~\cite{ghasemipour2021emaq, beirami2024theoretical, yang2024asymptotics, xureinforcement} as follows:
\begin{equation}
    a^*_i = \arg\max_{a^k_i \sim \mu_i(\cdot|o_i), k \in [N]} Q_i(o_i, a^k_i).
    \label{eq: best-of-n}
\end{equation}
This requires no policy training in terms of value maximization beyond behavioral policy modeling.

\textbf{Flow-matching and behavioral policy modeling.} To model behavioral policy in a multi-agent system, generative modeling offers a flexible framework that can capture multi-modal distributions~\cite{zhu2024madiff, li2025dof, lee2025multi, li2025om2p, hukmagic}. Specifically, flow-matching can relax the tradeoff between coordination performance and computational cost in MARL. Flow-matching~ learns a time-dependent velocity field $v_{\phi}(t,x):[0,1]\times\mathbb{R}^d\rightarrow\mathbb{R}^d$ that transports noise $p_0=\mathcal{N}(0, I_d)$ to the data distribution $p_1$~\cite{lipman2022flow}. Given the conditional linear path $x_t=(1-t)x_0 + tx_1$, the training objective is as follows:
\begin{equation}
\mathcal{L}_{\mathrm{FlowMatching}}(\phi) = \mathbb{E}_{x_0 \sim p_0, x_1 \sim p_1, t \sim \mathrm{Unif}([0,1])}\!\bigl[\lVert v_\phi(t, x_t) - (x_1 - x_0)\rVert_2^2\bigr]. \nonumber
\end{equation}
At inference, a sample is produced by integrating the ODE from initial noise~\cite{papamakarios2021normalizing, wildberger2023flow, lipman2024flow}.

Under multi-agent learning, different structural assumptions on the velocity field yield a hierarchy of policy representations~\cite{li2025om2p, lee2025multi}. A joint flow policy $\mu_\phi(\mathbf{o}, \mathbf{z})$ conditions the velocity field on the full joint observation $\mathbf{o}$ and joint noise $\mathbf{z} \sim \prod_i \mathcal{N}(0, I_{d_i})$, capturing inter-agent dependencies, as follows:
\begin{equation}
    \mathcal{L}_{\mathrm{JointFlow}}(\phi) = \mathbb{E}_{\mathbf{z} \sim p_0, ~(\mathbf{o}, \mathbf{a}) \sim \mathcal{D}, ~t\sim\text{Unif}([0,1])}\bigl[\lVert v_\phi(t, \mathbf{o}, x_t) - (\mathbf{a} - \mathbf{z})\rVert_2^2\bigr].
    \label{eq: joint-flow}
\end{equation}
While expressive, the joint policy violates CTDE at execution time. Instead, a factorized model trains individual policy $\hat{p}(\mathbf{a} \mid \mathbf{o}) = \prod_i \hat{p}_i(a_i \mid o_i)$~\cite{li2025dof}. This is fully decentralized but discards coordination patterns. A projected model~\cite{lee2025multi, hukmagic} bridges the two by learning a joint flow and then distilling it into the factorized one, $\hat{p}_{\mathrm{proj}} = \arg\min_{\hat{p} \in \mathcal{P}} \mathcal{L}(\hat{p} \| p)$, retaining as much coordination as the factorized structure permits. In practice, this can be written as follows:
\begin{equation}
    \mathcal{L}_{\text{ProjectedFlow}}(\mathbf{w}) = \mathbb{E}_{\mathbf{o}\sim \mathcal{D}, ~ \mathbf{z} \sim \mathbf{p}_0 } \Bigg[\sum_{i=1}^I \left|\left|\mu_{w_i}(o_i, z_i) - [\mu_\phi(\mathbf{o}, \mathbf{z})]_i \right|\right|^2_2\Bigg].
\end{equation}

\section{Scalable Coordination via Optimal Unified Transport (\texttt{SCOUT})}
\label{sec: scout}
As motivated in Section~\ref{sec: intro}, our goal is to develop an offline MARL algorithm that can extract Q-maximizing multi-agent policies without explicit policy parameterization for value maximization. Following prior works~\cite{xureinforcement, eysenbach2021maximum, bloem2014infinite}, we consider the MaxEnt surrogate of~\eqref{eq: policy}, $\mathbb{E}_{a \sim \pi}[Q(s,a)] + \alpha \mathcal{H}(\pi(\cdot|s))$, whose optimal policy takes the Boltzmann form $\pi^*(a|s) \propto \exp(Q(s,a)/\alpha)$.

Rather than fitting an explicit actor to approximate $\pi^*$, we construct a transport map that moves samples from $\mu$ toward $\pi^*$ via variational gradient descent, using $\nabla_a Q$ as the driving force. Specifically, given a set of action candidates sampled from a flow-matching behavioral policy, we iteratively refine them at test time via the simplest possible update as $a \leftarrow a + \eta \nabla_a Q(o, a)$. The resulting refined actions constitute the implicit policy output, with the number of gradient steps $L$ serving as a tunable knob that controls coordination strength at deployment without retraining.

\textbf{Why do we need an implicit policy?} This direction offers several appealing properties. First, because the initial actions are drawn from the behavioral policy, the implicit policy refines the behavioral prior rather than replacing it. Then, the transport budget provides an adaptive test-time scaling. When $L=0$, it recovers the best sampling from the behavioral policy, while increasing $L$ progressively pushes actions toward Q-maximizing regions. Furthermore, the gradient-based transport can discover actions beyond the reference distribution when the Q-function generalizes well.

\subsection{Why is this hard in multi-agent settings?}
The above idea and problem hold straightforwardly in the single-agent setting~\cite{xureinforcement, wu2026invisibleleashrlvrescape}, where the transport operates on a single action space $\mathcal{A}$ guided by a single Q-function. However, extending this to multi-agent settings is non-trivial. Under the CTDE paradigm, each agent $i$ can only access its local observation $o_i$ at test time. This means the action refinement must run independently for each agent, using only its local Q-function $Q_i(o_i, a_i)$ and its own action particles. The refined joint action $(a_1^*, \ldots, a_I^*)$ must be globally coherent. Namely, each agent's refined action should be coordinated with those of other agents. This tension between decentralized execution and coordinated behavior has been a long-standing challenge in MARL~\cite{son2019qtran, rashid2020monotonic}.

The most common solution is the IGM principle~\cite{son2019qtran} in~\eqref{eq: igm}, which ensures independent per-agent maximization to align with joint maximization. However, several practical questions remain. First, IGM guarantees alignment at the greedy argmax, but the gradient-based transport involves iterative updates that modify the particle distribution across multiple steps. It is unclear whether the per-agent transport dynamics remain aligned with the joint optimum throughout this process, especially when the critic must evaluate actions that are progressively shifted away from the behavioral support during training. Second, the individual behavioral prior might discard the inter-agent correlations present in the offline data. Even if the transport dynamics are correctly factorized, the initial particles may lack the coordination patterns necessary for effective joint behavior.

\subsection{Proposed solution}
In this subsection, we now present the full \texttt{SCOUT} framework for multi-agent test-time coordination.

\textbf{Flow-matching behavioral prior.} We learn a per-agent velocity field via flow matching~\cite{lipman2022flow}. 
\begin{equation}
    \mathcal{L}_\text{Flow-BC}(w_i) = \mathbb{E}_{x_i^0\sim p_0, ~ (o_i,a_i) \sim \mathcal{D}, ~ t\sim \text{Unif}([0,1])} \left[ \| v_{w_i}(t, o_i, x_i^t) - (a_i - x^0_i) \|_2^2 \right].
\end{equation}
This is a generative model of the behavioral policy $\mu_i(\cdot|o_i)$ that serves as the source distribution for the transport. At inference time, action samples are produced by Euler integration over $K$ steps.

\textbf{Decomposed value function.} The value function has to satisfy that it must decompose under IGM so that per-agent gradient transport is globally consistent. We achieve this by training $\{Q_{\theta_i}\}_{i=1}^I$ whose TD targets are computed through an additive decomposition as $Q_{\mathrm{tot}}(\mathbf{o}, \mathbf{a}) = \sum_{i=1}^{I} Q_{\theta_i}(o_i, a_i)$. This is the simplest decomposition satisfying IGM. The critic is trained via TD learning with targets $y = r + \gamma Q_{\mathrm{tot}}(\mathbf{o}', \mathbf{a}')$, where the next-state actions $a_i'$ are obtained by policy. This global credit signal is what enables per-agent gradient transport to produce coordinated behavior at test time. 

\textbf{Action refinement via decomposed gradient transport.} 
Given the pre-trained behavioral prior and the decomposed value function, we now introduce the test-time refinement as the core of \texttt{SCOUT}. Each agent $i$ begins by drawing $P$ action particles from its behavioral flow policy, $a_{i,j}^{(0)} = \mu_{w_i}(o_i, z_{i,j})$ with $z_{i,j} \sim \mathcal{N}(0, I)$ for $j = \{1, \ldots, P\}$. These particles are then iteratively refined for $L$ steps, motivated by Stein variational gradient descent (SVGD) as transport in Stein geometry~\cite{liu2016stein, liu2017stein} as follows:
\begin{equation}
    \hat{\phi}_i^*(a) = \frac{1}{P}\sum_{j=1}^{P}\left[\frac{1}{\alpha}k(a_{i,j}^{(\ell)}, a)\nabla_{a_{i,j}^{(\ell)}} Q_i(o_i, a_{i,j}^{(\ell)}) + \nabla_{a_{i,j}^{(\ell)}} k(a_{i,j}^{(\ell)}, a)\right]~~\text{s.t.}~~a_{i,j}^{(\ell+1)} = a_{i,j}^{(\ell)} + \epsilon\hat{\phi}_i^{\mathrm{clip}}(a_{i,j}^{(\ell)}),
    \label{eq: svgd}
\end{equation}
where $k(x,y) = \exp(-\gamma\|x - y\|^2)$ is an RBF kernel with adaptive median bandwidth. The first term in~\eqref{eq: svgd} is a kernel-smoothed value gradient that drives particles toward high-$Q_i$ regions; the second term is a repulsive force that prevents particle collapse and preserves multi-modality. 

For discrete action spaces, the refinement operates in the logit space. The flow policy produces logit vectors $\ell_{i,j} \in \mathbb{R}^{|\mathcal{A}_i|}$, and the value gradient is computed as $\nabla_{\ell}\sum_a \mathrm{softmax}(\ell)_a \cdot Q_i(o_i, a)$ via automatic differentiation through the softmax-weighted Q-value.

\subsection{Theoretical analysis}
We now provide theoretical guarantees for the \texttt{SCOUT} framework. 

\textbf{Soft value gap under approximate IGM.} The learned decomposition is at best an approximation of the true joint value, so we adopt the following residual assumption.
\begin{assumption}
\label{ass: decompose}
    The global Q-function decomposes as $Q_{\mathrm{tot}}(\mathbf{o}, \mathbf{a}) = \sum_{i=1}^{I} Q_i(o_i, a_i) + \Delta(\mathbf{o}, \mathbf{a})$, where the residual interaction term satisfies $\|\Delta\|_\infty \leq \delta$.
\end{assumption}
When $\delta=0$, exact additive IGM holds and decentralized transport recovers the global Boltzmann optimum. The following result quantifies the soft-value gap when $\delta>0$.
\begin{theorem}[Soft value gap under approximate IGM]
\label{theorem: igm}
Let $\tilde{\pi}^L = \prod_i \tilde{\pi}_i^L$ be the joint policy after $L$ steps of independent per-agent transport, each targeting $\pi_i^* \propto \exp(Q_i/\alpha)$. Let $\pi_{\mathrm{joint}}^* \propto \exp(Q_{\mathrm{tot}}/\alpha)$ be the true joint Boltzmann target and $\mathcal{J}_{\mathrm{soft}}(\pi):=\mathbb{E}_\pi[Q_{\mathrm{tot}}]+\alpha\mathcal{H}(\pi)$. Then
\begin{equation}
0 \leq \mathcal{J}_{\mathrm{soft}}(\pi_{\mathrm{joint}}^*) - \mathcal{J}_{\mathrm{soft}}(\tilde{\pi}^L) \leq \alpha\underbrace{\mathrm{KL}\!\left(\tilde{\pi}^L \Big\| \textstyle\prod_i \pi_i^*\right)}_{\mathrm{transport~error}} + \underbrace{4\delta}_{\mathrm{IGM~violation}}.
\end{equation}
\end{theorem}

The transport error is a single $\mathrm{KL}$ term that vanishes as $L\to\infty$ under idealized population SVGD with exact Q-functions and 
suitable step sizes. The SVGD performs gradient descent on $\mathrm{KL}(\cdot \| \prod_i\pi_i^*)$ in the kernel-induced Stein geometry~\cite{liu2016stein, liu2017stein}. The $4\delta$ term decomposes as $2\delta$ from the per-agent transport target $\pi_i^*$ and $2\delta$ from the gap between the factorized and joint Boltzmann optima; both arise from the IGM decomposition itself and apply equally to all decomposed-Q methods, including \texttt{MAC-Flow}'s distillation~\cite{lee2025multi}. What distinguishes \texttt{SCOUT} is that the transport error can be driven down by increasing $L$ at test time without retraining, whereas distillation-based methods fix their approximation quality at training time.

\textbf{Decomposed transport approximates joint best-of-$P$.} Even though each agent's gradient update is entirely local, we show that the combination of decomposed transport and per-agent selection tracks the joint best-of-$P$ outcome up to the irreducible IGM gap.

\begin{proposition}[Decomposed transport under IGM]
\label{prop: transport}
Consider $P$ joint candidates with independently initialized particles $(a_{1,j}^{(0)},\ldots,a_{I,j}^{(0)}) \sim \prod_i \mu_i$. After $L$ steps of per-agent gradient transport: (\textit{i}) when $\Delta = 0$, per-agent selection $a_i^* = \arg\max_j Q_i(o_i, a_{i,j}^{(L)})$ yields $\mathrm{Cov}(a_i^*, a_{i'}^*) = 0$ and is jointly optimal; (\textit{ii}) when $\|\Delta\|_\infty\leq\delta$, per-agent selection approximates joint best-of-$P$ within the irreducible IGM gap as $Q_{\mathrm{tot}}(a_1^*,\ldots,a_I^*) \geq \max_{j\in[P]} Q_{\mathrm{tot}}\bigl(a_{1,j}^{(L)},\ldots,a_{I,j}^{(L)}\bigr) - 2\delta.$
\end{proposition}
Herein, (\textit{i}) follows since each agent's particles and selection depend only on its own $Q_i$ and initial noise. Then, (\textit{ii}) follows from $|Q_{\mathrm{tot}} - \sum_i Q_i| \leq \delta$ uniformly; full proofs are in Appendix~\ref{app:proof}.

\textbf{Implicit coordination via TD targets.} Proposition~\ref{prop: transport} shows that decomposed transport is safe, \textit{i.e.}, it does not lose more than the irreducible $2\delta$ relative to a hypothetical centralized refiner. The complementary question is whether the per-agent gradient $\nabla_{a_i} Q_i$ itself encodes inter-agent dependencies. Recall the TD target $y = r + \gamma\sum_j Q_j(o'_j, a'_j)$. At the TD fixed point, $Q_i$ inherits sensitivity to other agents' actions through the joint distribution of $a'_j$ used to form the target. Concretely, fitted $Q_i$ can be interpreted as approximating $\nabla_{a_i} Q_i(o_i, a_i) \approx \mathbb{E}_{a_{-i}\sim\rho}[\nabla_{a_i} Q_{\mathrm{tot}}(\mathbf{o},\mathbf{a})]$ along the training distribution $\rho$. Under additive decomposition, $\partial Q_{\mathrm{tot}}/\partial Q_i = 1$, so $\nabla_{a_i} Q_i$ preserves the sign of the joint gradient $\nabla_{a_i} Q_{\mathrm{tot}}$, and per-agent transport along $\nabla_{a_i} Q_i$ is consistent with joint Q-improvement. Coordination is thus inherited from training, not synthesized at test time, and decomposed transport amplifies it through repeated kernel-smoothed updates over the particle set.

\input{table/pseducode}

\subsection{Algorithm summary}

We summarize the \texttt{SCOUT} framework in Pseudocode~\ref{alg:scout}. Training alternates between two decoupled objectives: (\textit{i}) the flow-matching loss  $\mathcal{L}_{\mathrm{BC}}$ fits the individual behavioral prior with ODE in Pseudocode~\ref{alg:flow_ode}, and the TD loss $\mathcal{L}_{\mathrm{critic}}$ learns decomposed Q functions whose targets are computed through the additive decomposition $Q_{\mathrm{tot}} = \sum_i Q_i$ using actions refined by $L_{\mathrm{train}}$ transport steps in Pseudocode~\ref{alg:vgf}.

At test time, each agent independently generates $P$ behavioral candidates via ODE integration, refines them for $L_{\mathrm{test}}$ steps of decomposed value gradient flow, and selects the highest-$Q_i$ particle. The transport budget $(L_{\mathrm{test}}, \epsilon)$ need not match training. In particular, $L_{\mathrm{test}} = 0$ recovers best-of-$P$ sampling, $L_{\mathrm{test}} = L_{\mathrm{train}}$ matches training conditions, and $L_{\mathrm{test}} > L_{\mathrm{train}}$ enables additional test-time scaling when the Q-function generalizes beyond the data support.

\subsection{Didactic example: Why do we need test-time value guidance, not training time?}

To illustrate the trade-off between training-time and test-time value incorporation, we design a two-agent coordination game (Figure~\ref{fig: toy} (\textit{a})). Each agent selects a scalar action $a_i \in [-2, 2]$, and the joint reward has two anti-correlated Gaussian modes: an optimal mode at $(-1, 1)$ with reward $4$ and a sub-optimal mode at $(1, -1)$ with reward $2$. The offline dataset consists of 400 transitions, with only $10\%$ near each peak. Both methods share the same VDN critic $Q_{\mathrm{tot}} = Q_1(a_1) + Q_2(a_2)$.

The key difference is how each method incorporates the Q-function. MAC-Flow jointly optimizes distillation and Q-maximization, forcing the projected one-step policy to balance behavioral cloning and value improvement during training. In this toy example, this training-time compromise anchors each agent's output to the broad data distribution, while the decomposed Q-gradient alone cannot represent the anti-correlated structure, causing the policy to collapse near the origin (Figure~\ref{fig: toy} (\textit{c})). In contrast, SCOUT keeps the behavioral flow unchanged and applies value guidance only at test time. This preserves the multi-modal prior, and five refinement steps are sufficient to move candidates toward the optimal mode (Figure~\ref{fig: toy} (\textit{b})). In this setting, the selected per-agent actions align with the same anti-correlated mode because the additive utilities are trained from coordinated demonstrations.

\input{figure/toy}

\section{Experiments}
\label{sec: exp}
The subsection presents the experimental setup and answers the following research questions. 

{\textbf{{RQ1}. How good is \texttt{SCOUT} over continuous and discrete control domains of MARL testbeds?}

{\textbf{{RQ2}. How does \texttt{SCOUT} improve with additional interaction data?}

{\textbf{{RQ3}. How sensitive is \texttt{SCOUT} to the refinement step parameter $L_\text{train}$ and $L_\text{test}$?}

\input{table/smac_}

\subsection{Experimental setup}
We evaluate \texttt{SCOUT} on two cooperative MARL testbeds with discrete and continuous action spaces, spanning $7$ scenarios and $24$ offline datasets. We first introduce the selected testbeds and baselines.

\textbf{SMACv1}. Agents control a team of units in StarCraft II micromanagement tasks against a built-in AI opponent. We test on four maps, \textit{including} \texttt{3m}, \texttt{8m}, \texttt{2s3z}, and \texttt{5m\_vs\_6m}. Each map is paired with three offline dataset quality levels (\texttt{Good}, \texttt{Medium}, \texttt{Poor}) from the off-the-grid benchmark~\cite{formanek2023off}, resulting in $12$ configurations. The quality levels differ in the proportion of near-optimal trajectories, with Poor containing mostly suboptimal behavior.

\textbf{MA-MuJoCo.} In this testbed, a single robot is factorized into multiple agents, with each agent responsible for a specific subset of the robot's joints. This setup requires agents to learn coordinated behaviors to accomplish shared tasks. We source the corresponding offline datasets from~\cite{wang2023offline}.

\textbf{Baseline.} We organize baselines into three families by their policy type. For the Gaussian policy, \texttt{BC}, \texttt{MATD3BC}~\cite{fujimoto2021minimalist} , \texttt{MABCQ}~\cite{fujimoto2019off}, \texttt{MACQL}~\cite{kumar2020conservative}, and \texttt{ICQ}~\cite{yang2021believe} extend single-agent methods to the multi-agent setting. For the diffusion policy, \texttt{MADiff}~\cite{zhu2024madiff} models the joint policy as a diffusion process over the full action space, but we follow a distributed version (\texttt{MADiff-D}) that masks other agents' information. \texttt{DoF}~\cite{li2025dof} factorizes the diffusion into per-agent components to satisfy decentralized execution. For the flow policy, \texttt{Flow BC} trains a per-agent flow-matching behavioral prior, and \texttt{MAC-Flow}~\cite{lee2025multi} learns a joint flow and distills it into one-step projected policies with Q-maximization.

\input{table/mamujoco}

\subsection{Experimental results}
\textbf{RQ1 (Offline RL).} Tables~\ref{tab:disc_perf} and~\ref{tab: cont_perf} report offline RL performance on SMACv1 and MA-MuJoCo. \texttt{SCOUT} achieves the highest average score on both benchmarks, outperforming all baselines across policy classes. The gain is consistent, \textit{i.e.}, \texttt{SCOUT} matches or exceeds the best baseline on the majority of individual configurations. Notably, the advantage over \texttt{MAC-Flow}, which shares the same flow-matching backbone, isolates the benefit of test-time value transport over training-time distillation.

\input{figure/o2o}
\textbf{RQ2 (Offline-to-online RL).}
Figure~\ref{fig:o2o} evaluates how well each method transitions from offline pre-training to online fine-tuning. \texttt{SCOUT} outperforms \texttt{MAC-Flow} across all four scenarios. After the online transition, \texttt{SCOUT} continues to improve steadily. Because the behavioral prior is never deformed by a training-time value loss, it provides a stable foundation that online fine-tuning can build upon rather than correct. The advantage is most pronounced on Poor-quality datasets. Here, \texttt{MAC-Flow}'s offline policy is anchored to low-quality demonstrations by its distillation objective, and online fine-tuning struggles to overcome this initialization. \texttt{SCOUT}, by contrast, treats the offline prior as a sampling distribution rather than a fixed policy, so the online phase simply adjusts the transport budget to exploit newly available data without inheriting the limitations of the offline dataset.

\input{figure/sweep}
\textbf{RQ3 (Hyperparameter sensitivity).} Figure~\ref{fig:sweep} varies the refinement steps for training and test across four SMACv1-Medium scenarios. First, applying even a small number of refinement steps at test time improves over the no-refinement baseline ($L_\text{test}=0$), confirming that decomposed value gradient transport adds meaningful coordination beyond best-of-$P$ selection alone. Second, the optimal $L_\text{test}$ varies across scenarios. This is expected because each scenario and dataset quality presents a different Q-function landscape; the number of steps needed to reach high-value regions depends on how far the behavioral prior is from those regions and how smooth the Q-surface is. In practice, $L_\text{test}$ can be selected via a small validation sweep at deployment time without retraining, which is precisely the adaptivity that training-time methods cannot offer.

\section{Conclusion}
In this paper, we propose \texttt{SCOUT}, the first test-time policy refinement framework for offline MARL that decouples behavioral modeling from value maximization. While prior generative policies often compromise expressiveness by integrating value guidance during training, \texttt{SCOUT} preserves the multi-modal behavioral prior and defers value-driven refinement strictly to test-time inference via decomposed Stein variational gradient transport. We provide a single-term KL bound on the joint soft-value gap under approximate IGM and constructively explain how additive TD targets transfer inter-agent dependencies into per-agent gradients.

\textbf{Limitation.} While decoupling the behavioral prior from value maximization is effective across a range of tasks, the refinement hyperparameters exhibit some task- and dataset-dependent sensitivity. Therefore, a principled adaptive stopping criterion remains an open direction for future work.

\section*{Acknowledgements}
This work was supported by the National Science Foundation Graduate Research Fellowship Program under Grant No. NSF 2340651, NSF 2402650, and ARO W911NF-24-1-0193. This research used the Texas Advanced Computing Center (TACC) at The University of Texas at Austin for providing computational resources.
  
{
\small
\bibliographystyle{unsrt}
\bibliography{main}
}

\newpage
\appendix
\input{appendix/appendix}

\newpage
\input{appendix/checklist}

\end{document}

%% file: figure/overview.tex
\begin{figure}[h]
    \vspace{-0.3cm}
    \centering
    \includegraphics[width=\linewidth]{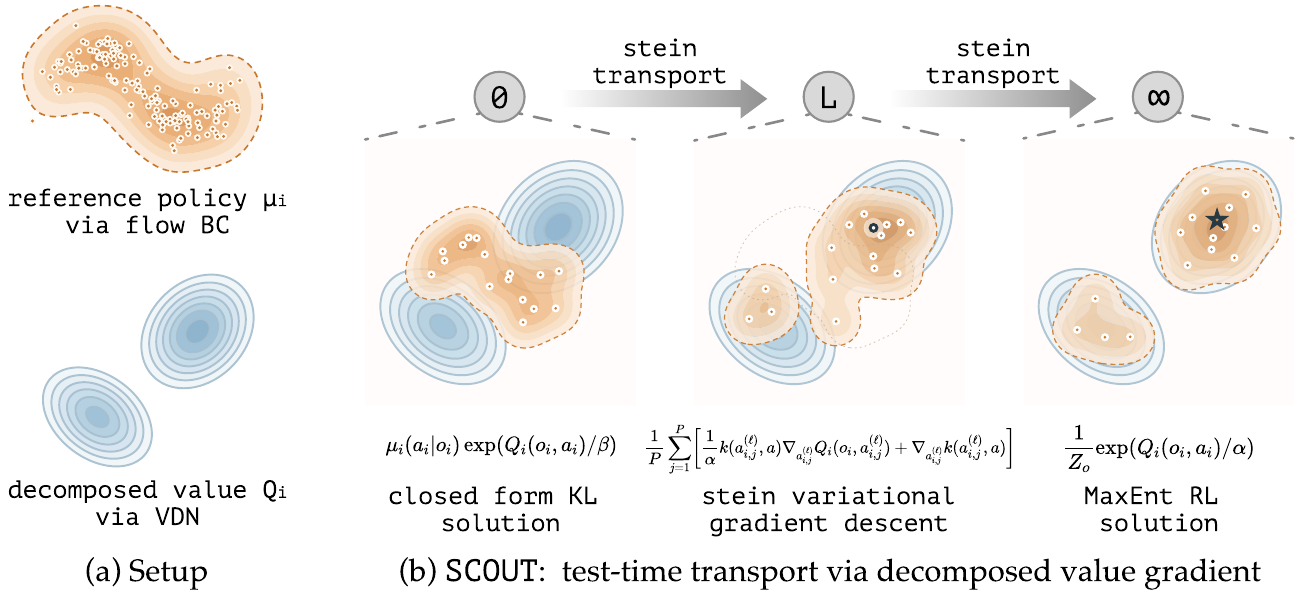}
    \vspace{-0.5cm}
    \caption{\textbf{Overview of the proposed solution.} (\textit{a}) The prerequisite for \texttt{SCOUT} is to train the reference policy via flow matching and decomposed value via a value decomposition network (VDN). (\textit{b}) \texttt{SCOUT} transports the behavior distribution at test time via Stein variational gradient descent.}
    \label{fig:overview}
\end{figure}

%% file: table/pseducode.tex
\begin{figure}[t]
\begin{minipage}[t]{0.52\textwidth}
\begin{algorithm}[H]
\caption{\texttt{SCOUT}'s training \& evaluation}
\label{alg:scout}
\begin{algorithmic}
    \fontsize{8pt}{10pt}\selectfont
    \State \textbf{Initialize} $v_w$, $\{Q_{\theta_i}\}$, $\{\bar{Q}_{\bar{\theta}_i}\}$
    \While{not converged}
        \State Sample $(o_i, a_i, r, o_i')_{i=1}^I \sim \mathcal{D}$
        \State
        \State {\color{gray}\textsc{\# Individual behavioral foundation}}
        \State $x_i^0 \sim \mathcal{N}(0, I)$, \; $t \sim \mathrm{Unif}([0,1])$
        \State $x_i^t \gets (1{-}t)\,x_i^0 + t\,a_i$
        \State $\mathcal{L}_{\mathrm{BC}} \gets \| v_w(t, o_i, x_i^t) - (a_i - x_i^0) \|^2$
        \State
        \State {\color{gray}\textsc{\# Decomposed value learning}}
        \State $\{a_i'\} \gets \textsc{{\color{utburnt}Refinement}}(o_i', v_w, \bar{Q}, L_{\mathrm{train}})$ \Comment{\color{darkgray}Alg.~\ref{alg:vgf}}
        \State $y \gets r + \gamma \sum_i(\bar{Q}_{\bar{\theta}_i}(o_i', a_i'))$
        \State $\mathcal{L}_{\mathrm{critic}} \leftarrow \tfrac{1}{2}|\sum_i Q_{\theta_i}(o_i,a_i) - y|^2$
    \EndWhile
    \State
    \State {\color{gray}\textsc{\# Evaluation}}
    \While{episode not done}
        \For{agent $i = 1, \ldots, I$ \textbf{in parallel}}
            \State $\{a_{i,j}^{(0)}\}_{j=1}^P \gets \textsc{{\color{myblue}ODE}}(o_i, v_w, P)$ \Comment{\color{darkgray}Alg.~\ref{alg:flow_ode}}
            \State $\{a_{i,j}^{(L)}\} \gets \textsc{{\color{utburnt}Refinement}}(o_i, Q_{\theta_i}, \{a_{i,j}^{(0)}\}, L_{\mathrm{test}})$ \Comment{\color{darkgray}Alg.~\ref{alg:vgf}}
            \State $a_i^* \gets \arg\max_j\, Q_{\theta_i}(o_i, a_{i,j}^{(L)})$
        \EndFor
        \State Execute $(a_1^*, \ldots, a_I^*)$, observe $(o_i', r, d)$
    \EndWhile
\end{algorithmic}
\end{algorithm}
\end{minipage}%
\hfill
\begin{minipage}[t]{0.46\textwidth}
\begin{algorithm}[H]
\caption{ODE solver}
\label{alg:flow_ode}
\begin{algorithmic}
    \fontsize{8pt}{10pt}\selectfont
    \Require observation $o_i$, velocity field $v_w$, particles $P$
    \Function{{\color{myblue}ODE}}{$o_i, v_w, P$}
        \State $a \gets z \sim \mathcal{N}(0, I)$ \Comment{$P$ particles}
        \For{$k = 0, \ldots, M{-}1$}
            \State $a \gets a + \tfrac{1}{M}\, v_w\!\bigl(\tfrac{k}{M},\, o_i,\, a\bigr)$
        \EndFor
        \State \Return $a$ \Comment{behavioral samples}
    \EndFunction
\end{algorithmic}
\end{algorithm}

\vspace{2em}

\begin{algorithm}[H]
\caption{\texttt{SCOUT}'s refinement}
\label{alg:vgf}
\begin{algorithmic}
    \fontsize{8pt}{10pt}\selectfont
    \Require observation $o_i$, critic $Q_i$, particles $\{a_{i,j}\}$, refinement steps $L$, refinement step size $\epsilon$
    \Function{{\color{utburnt}Refinement}}{$o_i, Q_i, \{a_{i,j}\}, L$}
        \State $a_{i,j}^{(0)} \gets a_{i,j}$ 
        \For{$\ell = 0, \ldots, L{-}1$}
            \State $\hat{\phi}_i^* \!\gets\! \tfrac{1}{P}\!\textstyle\sum_j \!\bigl[\tfrac{1}{\alpha}\,k(a_j, \cdot)\,\nabla_{\!a_j}\! Q_i \!+\! \nabla_{\!a_j}\! k(a_j, \cdot)\bigr]$
            \State $a_{i,j} \gets a_{i,j} + \epsilon\,\mathrm{clip}(\hat{\phi}_i,\, c)$
        \EndFor
        \State \Return $\{a_{i,j}\}$ \Comment{refined particles}
    \EndFunction
\end{algorithmic}
\end{algorithm}
\end{minipage}
\end{figure}

%% file: figure/toy.tex
\begin{figure}[t]
    \centering
    \includegraphics[width=\linewidth]{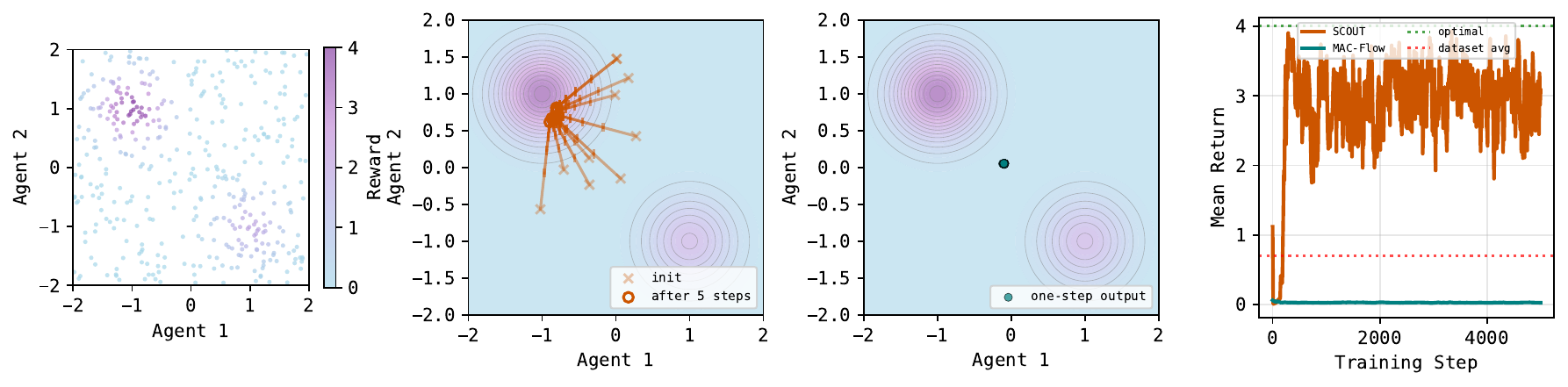}
    \vspace{-0.3cm}
    \caption{\textbf{Toy example.} (\textit{a}) Offline dataset. (\textit{b}) \texttt{SCOUT}'s action refinement ($L_\text{test}=5$). (\textit{c}) \texttt{MAC-Flow}'s one-step distilled policy. (\textit{d}) Learning curve of two algorithms. We plot $10$ evaluation examples.}
    \label{fig: toy}
\end{figure}

%% file: table/smac_.tex
\begin{table}[t]
    \centering
    \caption{A performance comparison for SMACv1 is provided, covering $4$ tasks, and $12$ datasets. Results are averaged over $8$ seeds, with two standard deviations reported after the $\pm$ sign. The superscript $^{\dagger}$ denotes reported scores from~\cite{formanek2023off}.}
    \label{tab:disc_perf}
    \scriptsize
    \resizebox{\textwidth}{!}{
    {\ttfamily
    \begin{tabular}{l l l c c c c c c c c}
    \toprule
    \multicolumn{2}{c}{\multirow{2}{*}{Scenarios}} & \multirow{2}{*}{Dataset}
    & \multicolumn{3}{c}{Gaussian policies}
    & \multicolumn{2}{c}{Diffusion policies}
    & \multicolumn{2}{c}{Flow policies}
    & \multicolumn{1}{c}{w/o policy} \\
    \cmidrule(lr){4-6} \cmidrule(lr){7-8} \cmidrule(lr){9-10} \cmidrule(lr){11-11}
    & & & BC$^{\dagger}$ & MABCQ$^{\dagger}$ & MACQL$^{\dagger}$ & MADiff-D & DoF & Flow BC & MAC-Flow & SCOUT \\
    \midrule

    {\multirow{13}{*}{\rotatebox[origin=c]{90}{SMACv1}}}
    & \multirow{3}{*}{3m} & Good
    & 16.0 \tiny{$\pm$ 1.0} & 16.3 \tiny{$\pm$ 1.5} & 19.6 \tiny{$\pm$ 0.3}
    & 19.3 \tiny{$\pm$ 0.5} & 19.4 \tiny{$\pm$ 0.1}
    & \textcolor{utburnt}{20.0} \tiny{$\pm$ 0.0} & 19.8 \tiny{$\pm$ 1.0} & \textcolor{utburnt}{20.0} \tiny{$\pm$ 0.0} \\

    & & Medium
    & 8.2 \tiny{$\pm$ 0.8} & 18.3 \tiny{$\pm$ 1.2} & \textcolor{utburnt}{18.9} \tiny{$\pm$ 1.2}
    & 16.4 \tiny{$\pm$ 2.6} & 18.2 \tiny{$\pm$ 0.2}
    & 14.7 \tiny{$\pm$ 1.5} & 16.8 \tiny{$\pm$ 1.3} & 17.8 \tiny{$\pm$ 2.2} \\

    & & Poor
    & 4.4 \tiny{$\pm$ 0.1} & 12.4 \tiny{$\pm$ 2.3} & 5.8 \tiny{$\pm$ 0.4}
    & 9.3 \tiny{$\pm$ 2.1} & 8.9 \tiny{$\pm$ 0.5}
    & 4.5 \tiny{$\pm$ 0.1} & 6.8 \tiny{$\pm$ 1.1} & \textcolor{utburnt}{17.0} \tiny{$\pm$ 3.6} \\
    \cmidrule(lr){2-11}

    & \multirow{3}{*}{8m} & Good
    & 16.7 \tiny{$\pm$ 0.4} & 12.7 \tiny{$\pm$ 6.3} & 11.3 \tiny{$\pm$ 6.1}
    & 18.9 \tiny{$\pm$ 1.1} & 19.1 \tiny{$\pm$ 0.1}
    & 19.5 \tiny{$\pm$ 0.2} & 19.4 \tiny{$\pm$ 1.3} & \textcolor{utburnt}{20.0} \tiny{$\pm$ 0.0} \\

    & & Medium
    & 10.7 \tiny{$\pm$ 0.5} & 16.0 \tiny{$\pm$ 1.4} & 16.8 \tiny{$\pm$ 3.1}
    & 16.8 \tiny{$\pm$ 1.6} & 17.5 \tiny{$\pm$ 1.2}
    & 18.2 \tiny{$\pm$ 0.8} & 19.1 \tiny{$\pm$ 1.0} & \textcolor{utburnt}{19.6} \tiny{$\pm$ 0.4} \\

    & & Poor
    & 5.3 \tiny{$\pm$ 0.1} & 5.8 \tiny{$\pm$ 1.4} & 4.6 \tiny{$\pm$ 2.4}
    & 5.8 \tiny{$\pm$ 0.9} & \textcolor{utburnt}{9.2} \tiny{$\pm$ 0.3}
    & 4.9 \tiny{$\pm$ 0.1} & 9.1 \tiny{$\pm$ 1.2} & 8.9 \tiny{$\pm$ 1.8} \\
    \cmidrule(lr){2-11}

    & \multirow{3}{*}{2s3z} & Good
    & 18.2 \tiny{$\pm$ 0.4} & 16.6 \tiny{$\pm$ 1.2} & 19.0 \tiny{$\pm$ 0.8}
    & 18.9 \tiny{$\pm$ 1.2} & 18.2 \tiny{$\pm$ 0.3}
    & 19.5 \tiny{$\pm$ 0.1} & 19.5 \tiny{$\pm$ 1.0} & \textcolor{utburnt}{19.9} \tiny{$\pm$ 0.1} \\

    & & Medium
    & 12.3 \tiny{$\pm$ 0.7} & 13.6 \tiny{$\pm$ 1.5} & 14.3 \tiny{$\pm$ 2.0}
    & \textcolor{utburnt}{17.5} \tiny{$\pm$ 0.9} & \textcolor{utburnt}{17.5} \tiny{$\pm$ 1.1}
    & 15.1 \tiny{$\pm$ 2.0} & 16.0 \tiny{$\pm$ 1.2} & \textcolor{utburnt}{17.5} \tiny{$\pm$ 1.9} \\

    & & Poor
    & 6.7 \tiny{$\pm$ 0.3} & \textcolor{utburnt}{11.5} \tiny{$\pm$ 1.0} & 10.1 \tiny{$\pm$ 0.7}
    & 9.8 \tiny{$\pm$ 1.3} & 9.8 \tiny{$\pm$ 1.6}
    & 6.9 \tiny{$\pm$ 0.8} & 8.4 \tiny{$\pm$ 1.1} & 10.1 \tiny{$\pm$ 1.1} \\
    \cmidrule(lr){2-11}

    & \multirow{3}{*}{5m\_vs\_6m} & Good
    & 16.6 \tiny{$\pm$ 0.6} & 8.3 \tiny{$\pm$ 0.9} & 13.8 \tiny{$\pm$ 3.9}
    & 16.5 \tiny{$\pm$ 2.8} & \textcolor{utburnt}{17.7} \tiny{$\pm$ 1.3}
    & 14.7 \tiny{$\pm$ 2.1} & 17.3 \tiny{$\pm$ 1.0} & 16.8 \tiny{$\pm$ 2.0} \\

    & & Medium
    & 12.4 \tiny{$\pm$ 0.9} & 12.1 \tiny{$\pm$ 1.3} & \textcolor{utburnt}{16.9} \tiny{$\pm$ 1.2}
    & 15.2 \tiny{$\pm$ 2.6} & 16.2 \tiny{$\pm$ 0.9}
    & 12.8 \tiny{$\pm$ 0.8} & 16.6 \tiny{$\pm$ 1.3} & 15.1 \tiny{$\pm$ 1.7} \\

    & & Poor
    & 7.5 \tiny{$\pm$ 0.2} & 11.0 \tiny{$\pm$ 0.9} & 10.4 \tiny{$\pm$ 1.0}
    & 8.9 \tiny{$\pm$ 1.3} & \textcolor{utburnt}{10.8} \tiny{$\pm$ 0.3}
    & 7.7 \tiny{$\pm$ 0.8} & 10.3 \tiny{$\pm$ 1.2} & 10.5 \tiny{$\pm$ 1.1} \\
    \cmidrule(lr){2-11}

    & \multicolumn{2}{c}{Average rewards}
    & 11.2 & 12.9 & 13.5 & 14.4 & 15.2 & 13.2 & 14.9 & \textcolor{utburnt}{16.1} \\
    \bottomrule
    \end{tabular}
    }}
\end{table}

%% file: table/mamujoco.tex
\begin{table}[t]
    \centering
    \caption{{Performance evaluation for MA-MuJoCo}. We present a performance comparison across $3$ tasks and $12$ datasets. The superscript $^{\dagger}$ denotes reported scores from~\cite{lee2025multi}.}
    \vspace{-0.05in}
    \label{tab: cont_perf}
    \scriptsize
    \resizebox{\textwidth}{!}{%
    {\ttfamily
    \begin{tabular}{l l l c c c c c c c }
    \toprule
    \multicolumn{2}{c}{\multirow{2}{*}{{Scenarios}}}  & \multirow{2}{*}{{Dataset}} &
    \multicolumn{4}{c}{\texttt{Offline MARL}} &
    \multicolumn{3}{c}{\texttt{Generative modeling}}\\
    \cmidrule(lr){4-7} \cmidrule(lr){8-10}
    & & & \texttt{MATD3BC}$^{\dagger}$ & \texttt{MACQL}$^{\dagger}$ & \texttt{ICQ}$^{\dagger}$ & \texttt{OMIGA}$^{\dagger}$ & {\texttt{MADiff-D}} & \texttt{MAC-Flow} & SCOUT \\
    \midrule
    {\multirow{13}{*}{\STAB{\rotatebox[origin=c]{90}{{MA-MuJoCo}}}}}
    & \multirow{4}{*}{\shortstack[l]{{Ant}}}
        & {E}
        & 93.8 \tiny{$\pm$ 1.4}
        & {\color{utburnt}{95.8}}\tiny{$\pm$ 1.0}
        & 93.9 \tiny{$\pm$ 0.9}
        & {94.2 \tiny{$\pm$ 0.1}}
        & 94.5 \tiny{$\pm$ 0.8}
        & {{94.5}} \tiny{$\pm$ 0.9}
        & {\color{utburnt}95.8} \tiny{$\pm$ 1.0}\\
    & & {M}
        & 42.4 \tiny{$\pm$ 1.0}
        & {11.2 \tiny{$\pm$ 1.1}}
        & 41.8 \tiny{$\pm$ 1.2}
        & {42.1 \tiny{$\pm$ 0.4}}
        & {{43.0}\tiny{$\pm$ 0.9}}
        & {{43.4}} \tiny{$\pm$ 0.7} 
        & {\color{utburnt}43.6} \tiny{$\pm$ 0.7}\\
    & & {M-R}
        & 8.0 \tiny{$\pm$ 4.2}
        & {-36.8 \tiny{$\pm$ 8.7}}
        & 9.8 \tiny{$\pm$ 4.3}
        & {17.1 \tiny{$\pm$ 7.2}}
        & {{31.6}\tiny{$\pm$ 29.1}}
        & {\color{utburnt}{43.7}} \tiny{$\pm$ 1.4} 
        & 42.5 \tiny{$\pm$ 4.6}\\
    & & {M-E}
        & 59.2 \tiny{$\pm$ 6.1}
        & {{73.5} \tiny{$\pm$ 0.9}}
        & 55.5 \tiny{$\pm$ 6.8}
        & {66.0 \tiny{$\pm$ 8.9}}
        & 67.7 \tiny{$\pm$ 12.7}
        & {{93.3}} \tiny{$\pm$ 0.6} 
        & {\color{utburnt}94.7} \tiny{$\pm$ 0.3}\\
    \cmidrule(lr){2-10} 
    & \multirow{4}{*}{\shortstack[l]{{Hopper}}}
        & {E}
        & 87.8 \tiny{$\pm$ 0.1}
        & {{88.9} \tiny{$\pm$ 13.6}}
        & 18.5 \tiny{$\pm$ 19.9}
        & {21.3 \tiny{$\pm$ 17.3}}
        & 74.6 \tiny{$\pm$ 14.3}
        & {\color{utburnt}95.2} \tiny{$\pm$ 0.2} 
        & {90.2} \tiny{$\pm$ 0.4}
        \\
    & & {M}
        & 21.7 \tiny{$\pm$ 3.8}
        & {22.5 \tiny{$\pm$ 4.9}}
        & 11.7 \tiny{$\pm$ 0.3}
        & {{30.3} \tiny{$\pm$ 13.2}}
        & {{37.0} \tiny{$\pm$ 10.9}}
        & 25.8 \tiny{$\pm$ 6.1}
        & {\color{utburnt}67.6} \tiny{$\pm$ 8.8}\\
    & & {M-R}
        & 5.4 \tiny{$\pm$ 1.0}
        & {-1.1 \tiny{$\pm$ 0.4}}
        & 3.4 \tiny{$\pm$ 2.5}
        & {18.8 \tiny{$\pm$ 12.0}}
        & {{23.4} \tiny{$\pm$ 13.9}}
        & {29.6} \tiny{$\pm$ 10.9}
        & {\color{utburnt}53.0} \tiny{$\pm$ 9.2}\\
    & & {M-E}
        & {\color{utburnt}76.6} \tiny{$\pm$ 11.6}
        & {72.7 \tiny{$\pm$ 3.0}}
        & 7.7 \tiny{$\pm$ 9.1}
        & {17.3 \tiny{$\pm$ 14.5}}
        & 74.0 \tiny{$\pm$ 17.5}
        & {73.0} \tiny{$\pm$ 11.7}
        & 75.5 \tiny{$\pm$ 10.7}
        \\
    \cmidrule(lr){2-10}
    & \multirow{4}{*}{\shortstack[l]{{HalfCheetah}}}
        & {E}
        & 111.3 \tiny{$\pm$ 4.1}
        & {116.1 \tiny{$\pm$ 2.4}}
        & 76.6 \tiny{$\pm$ 11.2}
        & {87.5 \tiny{$\pm$ 13.5}}
        & {\color{utburnt}{120.5}}\tiny{$\pm$ 5.2}
        & {{113.3}} \tiny{$\pm$ 6.6}
        & {111.6} \tiny{$\pm$ 7.9}\\
    & & {M}
        & 68.3 \tiny{$\pm$ 1.7}
        & {82.3 \tiny{$\pm$ 7.5}}
        & 66.6 \tiny{$\pm$ 2.3}
        & {{{93.0}} \tiny{$\pm$ 5.8}}
        & 69.1 \tiny{$\pm$ 8.9}
        & {\color{utburnt}{100.5}} \tiny{$\pm$ 9.0} 
        & {88.9} \tiny{$\pm$ 10.3} \\
    & & {M-R}
        & {{90.7}} \tiny{$\pm$ 2.9}
        & {{90.1} \tiny{$\pm$ 7.2}}
        & 51.5 \tiny{$\pm$ 14.9}
        & {65.5 \tiny{$\pm$ 2.0}}
        & 73.3 \tiny{$\pm$ 7.1}
        & 78.8 \tiny{$\pm$ 10.6} 
        & {\color{utburnt}95.1} \tiny{$\pm$ 6.6}\\
    & & {M-E}
        & 90.4 \tiny{$\pm$ 9.3}
        & {{120.0} \tiny{$\pm$ 4.4}}
        & 74.2 \tiny{$\pm$ 10.2}
        & {76.9 \tiny{$\pm$ 12.6}}
        & 113.6 \tiny{$\pm$ 20.4}
        & {{125.4}} \tiny{$\pm$ 2.0} 
        & {\color{utburnt}126.9} \tiny{$\pm$ 1.3} \\
    \cmidrule(lr){2-10}
    \multicolumn{3}{c}{{Average normalized score}} & 62.97 & 61.27 & 42.60 & 52.50 & 68.52 & 76.37 & {\color{utburnt}82.12} \\
    \bottomrule
    \end{tabular}}}
    \vspace{.15cm}
    {\scriptsize \raggedright Dataset abbreviations denote the data quality: \texttt{E (Expert), M (Medium), M-R (Medium-Replay)}, and \texttt{M-E (Medium-Expert)}. \par}
    \vspace{-.5cm} 
\end{table}

%% file: figure/o2o.tex
\begin{figure}[h]
    \centering
    \includegraphics[width=\linewidth]{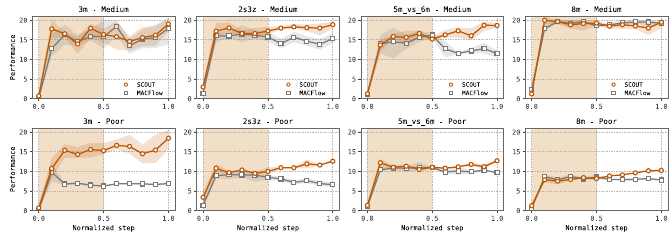}
    \vspace{-0.6cm}
    \caption{\textbf{Offline-to-online RL (8 seeds).} Fine-tuning starts at normalized gradient steps $0.5$.}
    \label{fig:o2o}
    \vspace{-0.2cm}
\end{figure}

%% file: figure/sweep.tex
\begin{figure}[h]
    \centering
    \includegraphics[width=\linewidth]{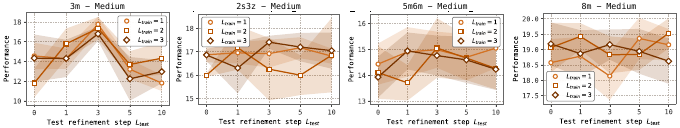}
    \vspace{-0.6cm}
    \caption{\textbf{Hyperparameter sweep.} The plots show how different $L$ affect offline RL performance.}
    \label{fig:sweep}
    \vspace{-0.2cm}
\end{figure}

%% file: appendix/appendix.tex
\par\noindent\rule{\textwidth}{2pt}
\begin{table}[h]
    \vspace{-0.3cm}
    \centering
    \Large
    \begin{tabular}{c}
\textbf{Appendix}
    \end{tabular}
    \vspace{-.5cm}
\end{table}
\par\noindent\rule{\textwidth}{1pt}

\tableofcontents
\addtocontents{toc}{\protect\setcounter{tocdepth}{2}}

\newpage

\section{Proofs}
\label{app:proof}
In this section, we provide complete proofs for the theoretical results in Section 4.3. Throughout, we use the following standing notation. Let $\mathcal{O} = \prod_{i=1}^{I} \mathcal{O}_i$ and $\mathcal{A} = \prod_{i=1}^{I} \mathcal{A}_i$ denote the joint observation and action spaces. Let $\mu_i(\cdot \mid o_i)$ denote the per-agent behavioral prior, and let $\pi_i^*(a_i \mid o_i) \propto \exp(Q_i(o_i, a_i)/\alpha)$ denote the per-agent factorized Boltzmann target. The joint Boltzmann target is defined as follows:
$$\pi_{\text{joint}}^*(a \mid o) \propto \exp\left(\frac{Q_{\text{tot}}(o, a)}{\alpha}\right).$$
We denote by $\tilde{\pi}^L = \prod_{i=1}^{I} \tilde{\pi}_i^L$ the joint policy obtained after $L$ steps of independent per-agent transport in~\eqref{eq: svgd}.

\subsection{Proof of Theorem~\ref{theorem: igm}}
\begin{proof}
Let $\pi_{\text{fact}}^* := \prod_i \pi_i^*$ and $\mathcal{J}_{\text{add}}(\pi) := \mathbb{E}_\pi[\sum_i Q_i] + \alpha\mathcal{H}(\pi)$. The product Boltzmann distribution $\pi_{\text{fact}}^*$ is the unique maximizer of $\mathcal{J}_{\text{add}}$ on the joint simplex (it is the Gibbs distribution for the additive utility), and the standard variational identity yields the gap-as-KL relation:
\begin{equation}
\label{eq: var-identity}
\mathcal{J}_{\text{add}}(\pi_{\text{fact}}^*) - \mathcal{J}_{\text{add}}(\pi) \;=\; \alpha\,\mathrm{KL}\!\bigl(\pi \,\|\, \pi_{\text{fact}}^*\bigr) \quad \text{for any joint } \pi.
\end{equation}
Under Assumption~\ref{ass: decompose}, $|\mathcal{J}_{\text{soft}}(\pi) - \mathcal{J}_{\text{add}}(\pi)| = |\mathbb{E}_\pi[\Delta]| \leq \delta$ for any $\pi$.

Combining \eqref{eq: var-identity} with the $\delta$-shift between $\mathcal{J}_{\text{soft}}$ and $\mathcal{J}_{\text{add}}$:
\begin{align}
\mathcal{J}_{\text{soft}}(\pi_{\text{fact}}^*) - \mathcal{J}_{\text{soft}}(\tilde{\pi}^L) 
&\leq \bigl[\mathcal{J}_{\text{add}}(\pi_{\text{fact}}^*) + \delta\bigr] - \bigl[\mathcal{J}_{\text{add}}(\tilde{\pi}^L) - \delta\bigr] \nonumber\\
&= \alpha\,\mathrm{KL}\!\bigl(\tilde{\pi}^L \,\|\, \pi_{\text{fact}}^*\bigr) + 2\delta. \label{eq: app-step1}
\end{align}

Since $\pi_{\text{joint}}^*$ maximizes $\mathcal{J}_{\text{soft}}$, $\mathcal{J}_{\text{soft}}(\pi_{\text{joint}}^*) - \mathcal{J}_{\text{soft}}(\pi_{\text{fact}}^*) \geq 0$. For the upper bound, since $\pi_{\text{fact}}^*$ maximizes $\mathcal{J}_{\text{add}}$, $\mathcal{J}_{\text{add}}(\pi_{\text{fact}}^*) \geq \mathcal{J}_{\text{add}}(\pi_{\text{joint}}^*)$, so
\begin{equation}
\mathcal{J}_{\text{soft}}(\pi_{\text{joint}}^*) - \mathcal{J}_{\text{soft}}(\pi_{\text{fact}}^*) \leq \bigl[\mathcal{J}_{\text{add}}(\pi_{\text{joint}}^*) + \delta\bigr] - \bigl[\mathcal{J}_{\text{add}}(\pi_{\text{fact}}^*) - \delta\bigr] \leq 2\delta. \label{eq: app-step2}
\end{equation}

Adding \eqref{eq: app-step1} and \eqref{eq: app-step2}:
$$\mathcal{J}_{\text{soft}}(\pi_{\text{joint}}^*) - \mathcal{J}_{\text{soft}}(\tilde{\pi}^L) \;\leq\; \alpha\,\mathrm{KL}\!\bigl(\tilde{\pi}^L \,\|\, \pi_{\text{fact}}^*\bigr) + 4\delta.$$
The lower bound $0$ follows directly from the $\mathcal{J}_{\text{soft}}$-optimality of $\pi_{\text{joint}}^*$.
\end{proof}

\subsection{Proof of Proposition~\ref{prop: transport}}
\begin{proof}
Mathematical derivations are divided into two parts of Proposition~\ref{prop: transport}. For \textbf{(\textit{i})}, the per-agent transport for agent $i$ as follows:
$$ a_{i,j}^{(\ell+1)} = a_{i,j}^{(\ell)} + \epsilon\,\mathrm{clip}_c\bigl(\hat{\phi}_i^*\bigl(a_{i,j}^{(\ell)}; \{a_{i,k}^{(\ell)}\}_{k=1}^P, Q_i(o_i, \cdot)\bigr)\bigr),$$
depends only on agent-$i$ particles and $Q_i$. Initialization is independent across agents and no cross-agent information enters any update, so independence is preserved:
\begin{equation}
  \label{eq: tra-p}
  \bigl(a_{i,j}^{(L)}\bigr)_{j=1}^P \perp\!\!\!\perp \bigl(a_{i',j}^{(L)}\bigr)_{j=1}^P \quad \text{for } i \neq i'. 
\end{equation}
Per-agent selection $j_i^* = \arg\max_j Q_i(o_i, a_{i,j}^{(L)})$ depends only on agent-$i$ particles, so $a_i^* = a_{i,j_i^*}^{(L)}$ are functions of independent sets, giving $\mathrm{Cov}(a_i^*, a_{i'}^*) = 0$. When $\Delta = 0$, $Q_{\mathrm{tot}} = \sum_i Q_i$, and per-agent argmax is jointly optimal by the IGM identity~\eqref{eq: igm}.

Next, for \textbf{(\textit{ii})}, let $j^* \in [P]$ be any joint candidate index. By per-agent optimality, $\sum_i Q_i(o_i, a_{i,j_i^*}^{(L)}) \geq \sum_i Q_i(o_i, a_{i,j^*}^{(L)})$. Using $Q_{\mathrm{tot}} = \sum_i Q_i + \Delta$ and $\|\Delta\|_\infty \leq \delta$:
\begin{align}
Q_{\mathrm{tot}}\bigl(a_{1,j_1^*}^{(L)}, \ldots, a_{I,j_I^*}^{(L)}\bigr) 
&= \textstyle\sum_i Q_i\bigl(a_{i,j_i^*}^{(L)}\bigr) + \Delta\bigl(a_{1,j_1^*}^{(L)}, \ldots, a_{I,j_I^*}^{(L)}\bigr) \nonumber \\
&\geq \textstyle\sum_i Q_i\bigl(a_{i,j^*}^{(L)}\bigr) - \delta \nonumber \\
&\geq Q_{\mathrm{tot}}\bigl(a_{1,j^*}^{(L)}, \ldots, a_{I,j^*}^{(L)}\bigr) - 2\delta. \label{eq: per-agent-gap}
\end{align}
Maximizing over $j^* \in [P]$ on the right-hand side yields the claim.
\end{proof}

\section{Experimental Details}
\label{app:exp}
\subsection{Computational resources}
All experiments were conducted on a single server equipped with an \texttt{Intel Xeon E5-2650 v3} CPU and an \texttt{NVIDIA A40} GPU. The software stack consists of \texttt{CUDA 12.2, cuDNN 8.9.7, Python 3.9, and JAX 0.4.30.} 

\subsection{Implementation details}

\textbf{Flow BC prior.} We instantiate the foundation as a conditional flow-matching model that learns the behavior distribution of the offline dataset. Given an observation and a Gaussian noise sample, the actor parameterizes a velocity field that transports $x_0$ to a data action $x_1$ along the linear interpolation, trained by minimizing the flow-matching loss. At inference time, candidate actions are produced by a fixed-step Euler integration of the learned ODE, so the prior remains an explicit, multimodal generative model of the dataset rather than a unimodal collapse. Because the resulting samples are biased toward the modeled behavioral support, they encourage the agent to remain near in-distribution behaviors, while still preserving the diverse coordination patterns present in the offline corpus. 

\textbf{Sinusoidal time embedding.} The flow conditioning variable $t\in[0,1]$ carries information about how close the current sample $x_t$ is to data; feeding it as a raw scalar concatenated to the network input yields poor expressivity, since a single-channel signal interacts weakly with high-dimensional weights and forces the actor to learn very different behaviors at neighboring $t$ values from a one-dimensional cue. We therefore lift $t$ into a Fourier feature space $\varphi(t)=\bigl[\sin(2^k\pi t),\cos(2^k\pi t)\bigr]_{k=0}^{K-1}\in\mathbb{R}^{2K}$ with $K=8$, echoing the positional encodings widely used in diffusion and neural-field models. The exponentially spaced frequencies provide a multi-resolution decomposition of the time axis: low-frequency components encode the coarse phase of the integration trajectory while high-frequency components let the velocity field model sharp transitions near $t=0$ and $t=1$. 

\textbf{Online fine-tuning.} After offline gradient steps, we roll out the policy and store the resulting trajectories in an online replay buffer. After an additional warm-up step with no online updates, each gradient step samples a $50:50$ mixture of offline data and newly collected online data. 

\subsection{Hyperparameters}
We select $L_\text{train}$ and $L_\text{test}$ from a small grid using a held-out validation split of the offline dataset; no environment interaction is used for selection. Baselines use their officially reported hyperparameters with comparable tuning budgets. 

\begin{table}[h]
    \centering
    \begin{minipage}{.48\textwidth}
        \centering
        \caption{Hyperparameter for RL optimization}
        \ttfamily\small
        \begin{tabular}{lc}
            \toprule
            \normalfont Hyperparameter & \normalfont Default \\
            \midrule
            Flow prior network & $(256, 256, 256, 256)$ \\
            Value prior network & $(256, 256, 256, 256)$ \\
            Layer normalization & value \\
            Learning rate & $0.0003$ \\
            Q value Aggregation & mean \\
            Flow steps & 10 \\
            $t$ embedding frequency & 8 \\
            $\tau$ & $0.005$ \\
            $\gamma$ & $0.99$ \\
            \bottomrule
        \end{tabular}
    \end{minipage}
    \hfill 
    \begin{minipage}{.48\textwidth}
        \centering
        \caption{Hyperparameter for \texttt{SCOUT}}
        \ttfamily\small
        \begin{tabular}{lc}
            \toprule
            \normalfont Hyperparameter & \normalfont Default \\
            \midrule
            $\epsilon$ refinement step size & $0.02$ \\
            $P$ $\#$ of particles & $10$ \\
            $L_\text{train}$ & $0$ \\
            $L_\text{test}$ & $0$ \\
            Particle selection (train) & mean \\
            Particle selection (test) & max \\
            $\alpha$ temperature & $1.0$ \\
            \bottomrule
        \end{tabular}
    \end{minipage}
\end{table}

\begin{table}[h]
    \centering
    \caption{Hyperparameter selection of \texttt{SCOUT} for offline RL configuration}
    \label{tab: default}
    \ttfamily{\begin{tabular}{lcc | lcc}
    \toprule
    \normalfont Configuration & $L_\text{train}$ & $L_\text{test}$ & \normalfont Configuration & $L_\text{train}$ & $L_\text{test}$\\
    \midrule
    3m - Good & $1$ & $1$ & 2s3z - Good & $1$ & $0$\\
    3m - Medium & $1$ & $3$ & 2s3z - Medium & $2$ & $1$ \\
    3m - Poor & $2$ & $3$ & 2s3z - Poor & $3$ & $10$ \\
    8m - Good & $3$ & $0$ & 5m\_vs\_6m - Good & $3$& $3$\\
    8m - Medium & $2$ & $10$ & 5m\_vs\_6m - Medium & $2$& $3$ \\
    8m - Poor & $3$ & $10$ & 5m\_vs\_6m - Poor &$1$& $3$ \\
    ant - Expert & $2$ & $0$ & ant - Medium-Expert & $0$ & $0$ \\
    ant - Medium & $2$ & $0$ & ant - Medium-Replay & $0$ & $0$ \\
    hopper - Expert & $2$ & $10$ & hopper - Medium-Expert & $2$ & $3$ \\
    hopper - Medium & $2$ & $3$ & hopper - Medium-Replay & $3$ & $0$ \\
    halfcheetah - Expert & $2$ & $3$ & halfcheetah - Medium-Expert & $3$ & $0$ \\
    halfcheetah - Medium & $0$ & $1$ & halfcheetah - Medium-Replay & $3$ & $1$ \\
\bottomrule
    \end{tabular}}
\end{table}

\subsection{Baselines}
\textbf{\texttt{MADiff-D}}~\cite{zhu2024madiff}. We utilize the official open-source implementation, specifically adopting the decentralized variant, \texttt{MADiff-D}, to align with our Dec-POMDP formulation. The model trains a conditional diffusion policy on centralized joint demonstrations but executes decentrally; each agent samples actions via reverse diffusion based solely on local observations while predicting teammates' trajectories. We retain the original architecture and hyperparameters, applying a Gaussian DDPM scheduler with $200$ denoising steps across all environments.

\textbf{\texttt{DoF}}~\cite{li2025dof}. We use the official implementation of the \texttt{DoF-Trajectory} agent equipped with the \texttt{W-concat} factorization. All training and deployment configurations strictly follow the original paper. For action sampling, we similarly employ a Gaussian DDPM scheduler with $200$ denoising steps for all environments.

\textbf{\texttt{MAC-Flow}}~\cite{lee2025multi}. We utilize the official open-source implementation. \texttt{MAC-Flow} models multi-modal data distributions via flow matching, utilizes a learned Q-function for guidance, and distills the guided flow into one-step decentralized policies. We maintain the original architecture and hyperparameters for all experiments.

\section{Additional Results}
\input{figure/o2o-continuous}
\textbf{Offline-to-online RL.} Figure~\ref{fig:o2o-con} extends the offline-to-online evaluation to the continuous control domain. Across all six configurations, SCOUT demonstrates consistent improvement after the online transition at normalized step $0.5$. These empirical results suggest that the undeformed flow prior allows the transport budget to flexibly exploit newly collected interaction data. Together with the SMACv1 results in Figure~\ref{fig:o2o}, \texttt{SCOUT} achieves positive online improvement in all offline-to-online RL configurations across both benchmarks, confirming that the test-time transport paradigm generalizes across discrete and continuous action space types.

\input{figure/sweep-con}
\textbf{Hyperparameter sweep.} Figure~\ref{fig: sweep-con} reports the sensitivity of \texttt{SCOUT} to the refinement steps $L_\text{train}$ and $L_\text{test}$ across all four dataset qualities in MA-MuJoCo. In the Expert and Medium-Expert datasets, performance is largely insensitive to the choice of $L_\text{test}$, which is expected because the behavioral prior already closely matches high-value regions, leaving little room for gradient transport to improve upon. In contrast, the Medium and Medium-Replay datasets exhibit a clearer benefit from moderate refinement steps, where increasing $L_\text{test}$ from 0 to 3 yields substantial gains. However, excessive refinement $L_\text{test} = 10$ can degrade performance in some cases, suggesting that the critic's accuracy degrades when actions are transported too far from the behavioral support.

%% file: figure/o2o-continuous.tex
\begin{figure}[h]
    \centering
    \includegraphics[width=\linewidth]{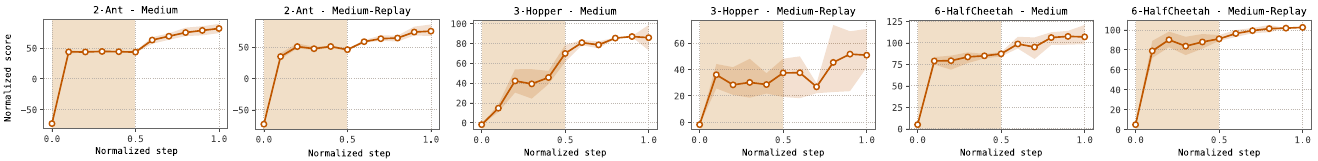}
    \caption{\textbf{Offline-to-online RL in MA-MuJoCo (8 seeds.)}}
    \label{fig:o2o-con}
\end{figure}

%% file: figure/sweep-con.tex
\begin{figure}[h]
    \centering
    \includegraphics[width=\linewidth]{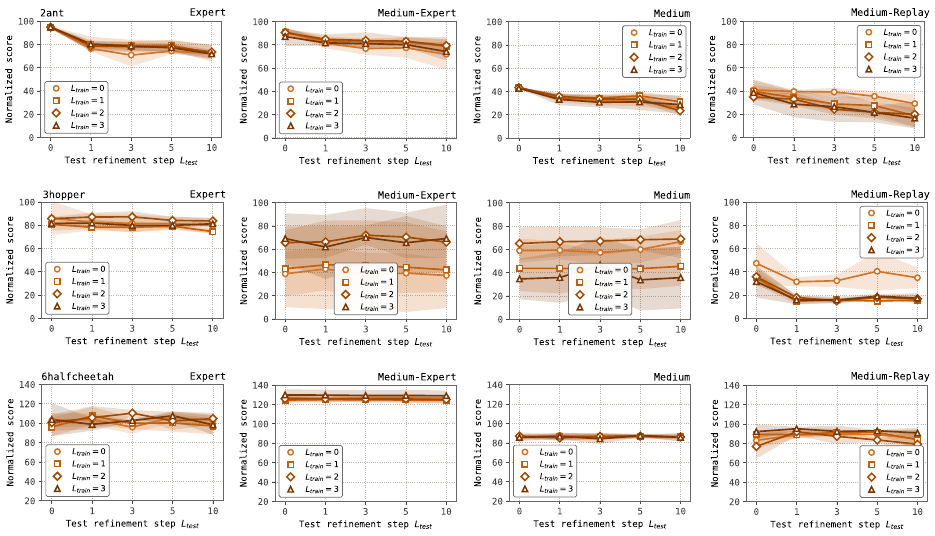}
    \caption{\textbf{Hyperparameter sweep on MA-MuJoCo.}}
    \label{fig: sweep-con}
\end{figure}

%% file: appendix/checklist.tex
\newpage
\section*{NeurIPS Paper Checklist}

\begin{enumerate}

\item {\bf Claims}
    \item[] Question: Do the main claims made in the abstract and introduction accurately reflect the paper's contributions and scope?
    \item[] Answer: \answerYes{}
    \item[] Justification: The abstract and introduction clearly state three contributions.

\item {\bf Limitations}
    \item[] Question: Does the paper discuss the limitations of the work performed by the authors?
    \item[] Answer: \answerYes{}
    \item[] Justification: The paper includes a dedicated limitation discussion in the Conclusion section.

\item {\bf Theory assumptions and proofs}
    \item[] Question: For each theoretical result, does the paper provide the full set of assumptions and a complete (and correct) proof?
    \item[] Answer: \answerYes{}
    \item[] Justification: The paper states Assumption~\ref{ass: decompose} explicitly and provides three theoretical results (Theorem~\ref{theorem: igm}, Proposition~\ref{prop: transport}) with complete proofs in Appendix~\ref{app:proof}. 

    \item {\bf Experimental result reproducibility}
    \item[] Question: Does the paper fully disclose all the information needed to reproduce the main experimental results of the paper to the extent that it affects the main claims and/or conclusions of the paper (regardless of whether the code and data are provided or not)?
    \item[] Answer:  \answerYes{}
    \item[] Justification: The paper provides detailed algorithm pseudocode (Algorithm~\ref{alg:scout}), complete hyperparameter tables, implementation details, and baseline descriptions in Appendix~\ref{app:exp}.

\item {\bf Open access to data and code}
    \item[] Question: Does the paper provide open access to the data and code, with sufficient instructions to faithfully reproduce the main experimental results, as described in supplemental material?
    \item[] Answer: \answerYes{}
    \item[] Justification: The paper provides an anonymous code repository link in the Abstract. The offline datasets used from the off-the-grid dataset~\cite{formanek2023off} are publicly available benchmarks.

\item {\bf Experimental setting/details}
    \item[] Question: Does the paper specify all the training and test details (e.g., data splits, hyperparameters, how they were chosen, type of optimizer, etc.) necessary to understand the results?
    \item[] Answer: \answerYes{}
    \item[] Justification: Section~\ref{sec: exp} describes the experimental setup, \textit{including} testbeds, datasets, and baselines. Appendix~\ref{app:exp} provides comprehensive implementation details, default and per-configuration hyperparameters, and baseline configurations.

\item {\bf Experiment statistical significance}
    \item[] Question: Does the paper report error bars suitably and correctly defined or other appropriate information about the statistical significance of the experiments?
    \item[] Answer:  \answerYes{}
    \item[] Justification: All results in Tables~\ref{tab:disc_perf} and~\ref{tab: cont_perf} report mean $\pm$ two standard deviations over $8$ random seeds.

\item {\bf Experiments compute resources}
    \item[] Question: For each experiment, does the paper provide sufficient information on the computer resources (type of compute workers, memory, time of execution) needed to reproduce the experiments?
    \item[] Answer: \answerYes{}
    \item[] Justification: The paper reports the type of GPU/CPU hardware used for the experiments.
    
\item {\bf Code of ethics}
    \item[] Question: Does the research conducted in the paper conform, in every respect, with the NeurIPS Code of Ethics \url{https://neurips.cc/public/EthicsGuidelines}?
    \item[] Answer: \answerYes{}
    \item[] Justification: The work conforms with the NeurIPS Code of Ethics.

\item {\bf Broader impacts}
    \item[] Question: Does the paper discuss both potential positive societal impacts and negative societal impacts of the work performed?
    \item[] Answer: \answerNA{}
    \item[] Justification: This work is foundational research on offline MARL. It does not have a direct path to negative societal applications.
    
\item {\bf Safeguards}
    \item[] Question: Does the paper describe safeguards that have been put in place for responsible release of data or models that have a high risk for misuse (e.g., pretrained language models, image generators, or scraped datasets)?
    \item[] Answer: \answerNA{}
    \item[] Justification: The paper does not pose risks for misuse comparable to pretrained language models or image generators.

\item {\bf Licenses for existing assets}
    \item[] Question: Are the creators or original owners of assets (e.g., code, data, models), used in the paper, properly credited and are the license and terms of use explicitly mentioned and properly respected?
    \item[] Answer: \answerYes{}
    \item[] Justification: The paper cites the original papers for all benchmarks and datasets used.

\item {\bf New assets}
    \item[] Question: Are new assets introduced in the paper well documented and is the documentation provided alongside the assets?
    \item[] Answer: \answerYes{}
    \item[] Justification: The paper releases code via an anonymous repository linked in the abstract.

\item {\bf Crowdsourcing and research with human subjects}
    \item[] Question: For crowdsourcing experiments and research with human subjects, does the paper include the full text of instructions given to participants and screenshots, if applicable, as well as details about compensation (if any)? 
    \item[] Answer: \answerNA{}
    \item[] Justification: This paper does not involve crowdsourcing or research with human subjects.

\item {\bf Institutional review board (IRB) approvals or equivalent for research with human subjects}
    \item[] Question: Does the paper describe potential risks incurred by study participants, whether such risks were disclosed to the subjects, and whether Institutional Review Board (IRB) approvals (or an equivalent approval/review based on the requirements of your country or institution) were obtained?
    \item[] Answer: \answerNA{}
    \item[] Justification: This paper does not involve crowdsourcing or research with human subjects.

\item {\bf Declaration of LLM usage}
    \item[] Question: Does the paper describe the usage of LLMs if it is an important, original, or non-standard component of the core methods in this research? Note that if the LLM is used only for writing, editing, or formatting purposes and does not impact the core methodology, scientific rigorousness, or originality of the research, declaration is not required.
    \item[] Answer: \answerNA{}
    \item[] Justification: The core method development in this research does not involve LLMs as any important, original, or non-standard components.

\end{enumerate}